\def\year{2021}\relax
\documentclass[letterpaper]{article} 
\usepackage{aaai21}  
\usepackage{times}  
\usepackage{helvet} 
\usepackage{courier}  
\usepackage[hyphens]{url}  
\usepackage{graphicx} 
\usepackage{natbib}  
\usepackage{caption} 
\usepackage{amsfonts}
\usepackage{amsmath}
\usepackage{graphicx}
\usepackage{xcolor}
\usepackage{amsthm}
\usepackage{booktabs}
\usepackage{bbm}
\usepackage[utf8]{inputenc}
\usepackage[english]{babel}
\definecolor{almond}{rgb}{0.99, 0.96, 0.89} 
\usepackage{listings}
\usepackage{algorithm}
\usepackage{mathrsfs}  
\usepackage{dsfont}

\usepackage{array}
\newcolumntype{L}[1]{>{\raggedright\let\newline\\\arraybackslash\hspace{0pt}}m{#1}}
\newcolumntype{C}[1]{>{\centering\let\newline\\\arraybackslash\hspace{0pt}}m{#1}}
\newcolumntype{R}[1]{>{\raggedleft\let\newline\\\arraybackslash\hspace{0pt}}m{#1}}

\newtheorem{theorem}{Theorem}
\newtheorem{lemma}{Lemma}
\newtheorem{corollary}{Corollary}[theorem]
\newtheorem{definition}{Definition}[section]

\newcommand{\comment}[1]{}
\newcommand{\aref}[2]{#2}

\title{A Theoretical Analysis of the Repetition Problem in Text Generation\thanks{The work described in this paper is substantially supported by a grant from the Research Grant Council of the Hong Kong Special Administrative Region, China (Project Code: 14204418).}}
\author {
    Zihao Fu,\textsuperscript{\rm 1} 
    Wai Lam,\textsuperscript{\rm 1}
    Anthony Man-Cho So,\textsuperscript{\rm 1}
    Bei Shi\textsuperscript{\rm 2} \\
}

\affiliations{
  \textsuperscript{\rm 1}Department of Systems Engineering and Engineering Management, \\The Chinese University of Hong Kong, Hong Kong\\ 
  \textsuperscript{\rm 2}AI Lab, Tencent\\
  \{zhfu, wlam, manchoso\}@se.cuhk.edu.hk; beishi@tencent.com
}

\begin{document}
\maketitle

\begin{abstract}
    Text generation tasks, including translation, summarization, language models, and etc. see rapid growth during recent years. Despite the remarkable achievements, the repetition problem has been observed in nearly all text generation models undermining the generation performance extensively. To solve the repetition problem, many methods have been proposed, but there is no existing theoretical analysis to show why this problem happens and how it is resolved. In this paper, we propose a new framework for theoretical analysis for the repetition problem. We first define the Average Repetition Probability (ARP) to characterize the repetition problem quantitatively. Then, we conduct an extensive analysis of the Markov generation model and derive several upper bounds of the average repetition probability with intuitive understanding. We show that most of the existing methods are essentially minimizing the upper bounds explicitly or implicitly. Grounded on our theory, we show that the repetition problem is, unfortunately, caused by the traits of our language itself. One major reason is attributed to the fact that there exist too many words predicting the same word as the subsequent word with high probability. Consequently, it is easy to go back to that word and form repetitions and we dub it as the high inflow problem. Furthermore, we extend our analysis to broader generation models by deriving a concentration bound of the average repetition probability for a general generation model. Finally,  based on the theoretical upper bounds, we propose a novel rebalanced encoding approach to alleviate the high inflow problem and thus reducing the upper bound.
    The experimental results show that our theoretical framework is applicable in general generation models and our proposed rebalanced encoding approach alleviates the repetition problem significantly in both the translation task and the language modeling task. The source code of this paper can be obtained from \url{https://github.com/fuzihaofzh/repetition-problem-nlg}.
    \end{abstract}
    
    \section{Introduction}

    Text generation tasks aim at generating human-readable text for specific tasks including machine translation~\cite{sutskever2014sequence,bahdanau2014neural,luong2015effective}, summarization~\cite{nallapati2016abstractive}, data-to-text generation~\cite{lebret2016neural,wiseman2017challenges,fu2020open,fu2020partially,fu2020dynamic}, language modeling~\cite{radfordimproving,radfordlanguage,brown2020language} and etc. Despite the remarkable growth, nearly all existing generation systems suffer from the repetition problem~\cite{fan2018hierarchical,holtzman2020curious,welleck2020neural} which unavoidably undermines the overall performance. The repetition problem refers to an undesirable effect that the results of the generation system always contain duplicate fragments. For example, a generation system is likely to generate text like ``Tough it is still unfinished, but I like it but I like it but I like ...'', which contains the useless repeating text. This problem hurts the quality of generation systems severely.
    
    There are many conjectures for the reason of the repetition problem, but the real cause is still unknown \cite{welleck2020neural}. Some conjectures think that it is caused by the model architectures \cite{holtzman2020curious,vig2018deconstructing}. Some of them suggest that it may be caused by the gap between sampling methods and the real human language \cite{holtzman2020curious}. \citet{choi2018missing} argues that reliance on the ﬁxed corpora cannot fulfill the real goal of using the language. Moreover, \citet{welleck2020neural} attribute the reason to the likelihood maximizing decoding. However, there is no existing theory to quantitatively explain how and why the repetition occurs. Based on these conjectures, many out-of-the-box methods have been proposed to alleviate the repetition problem. \citet{ackley1985learning,ficler2017controlling,caccia2019language} propose to utilize the Temperature sampling in the decoding phase. \citet{fan2018hierarchical} propose the Topk sampling method while \citet{holtzman2020curious} propose the Nucleus sampling method. Though these methods are shown to successfully alleviate the repetition problem, there still needs a convincing explanation.
    
    To determine the root cause of the repetition problem and give a convincing understanding of why existing methods work, we propose a novel theoretical analysis framework. We first define the general generation model and the Markov generation model as the core text generation models for our analysis. Then, we define the Average Repetition Probability (ARP) to characterize the repetition problem quantitatively on the Markov generation model. Subsequently, we conduct an extensive analysis of the Markov generation model and derive several upper bounds of the ARP with intuitive understanding. We show that most of the existing methods are essentially minimizing the theoretical upper bounds explicitly or implicitly. Guided by our theory, we show that the repetition problem is caused by, unfortunately, our language itself. One major reason is attributed to the fact that there exist too many words predicting the same word as the subsequent word with high probability. Consequently, it is easy to go back to that word and form repetitions and we dub it as the high inflow problem. We call the predicted word as high inflow word and call the previous word together with the high inflow word as high inflow pair. Furthermore, we extend our results to broader generation models by deriving a concentration bound of the average repetition probability for a general generation model. 
    
    Based on the theoretical upper bounds, we propose a novel rebalanced encoding approach to alleviate the high inflow problem and thus reducing the upper bound. We merge each high inflow pair as a new word and thus decrease the high inflow probability. We conduct experiments on two fundamental tasks namely the Neural Machine Translation (NMT) task and the Language Modeling (LM) task. The experimental results show that our proposed new method alleviates the repetition significantly and outperforms existing prevalent out-of-the-box methods. We also conduct several experiments to verify the correctness and applicability of our theory when applied in NMT and LM tasks.

    \section{Theoretical Framework}
    Our theoretical framework mainly encompasses three parts. We first define the Average Repetition Probability (ARP) to characterize the repetition problem quantitatively. Then, we analyze the repetition problem in a Markov generation model. We derive several upper bounds for ARP with intuitive understanding and we show that most of the existing approaches are trying to minimize the upper bounds explicitly or implicitly. Finally, we show that our analysis can be adopted into general generation models by proving that if each step's transition matrix does not deviate a lot from the Markov generation model, the deviation of the ARP value will be controlled with a concentration upper bound.

    \subsection{Definitions}
    \label{sec:def}
    To facilitate a quantitative analysis of the repetition problem, we first give some definitions which will be used in the later sections of this paper.
    We denote the generated sequence $s=[w_1,w_2,\cdots,w_{|s|}]$, in which $w_i$ is the $i$th word and $|s|$ is the length of $s$. We denote $s_{p:q}=[w_p,w_{p+1},\cdots,w_{q}],(1\leq p<q\leq |s|)$ as a continuous subsequence of $s$. We say a continuous subsequence $s_{r:t}$ is a \textbf{loop} if $w_r=w_t$ (it is allowed to contain sub-loops inside). We say a continuous subsequence $s_{p:q}=[s_p,s_{p+1},\cdots,s_{q}]$ is a \textbf{repetition subsequence} if $\exists 1\leq p<q\leq |s|-q+p,\ s.t. \  w_{i+q-p}=w_i, \forall i\in [p,q]$. It can be easily shown that a sequence has repetition subsequence if and only if it has at least two adjacent identical loops.
    As shown in Fig. \ref{fig:task} (a), the sequence contains repetition subsequences with adjacent loops while in Fig. \ref{fig:task} (b), it contains a loop without any repetition subsequence. It should be noted that in Fig. \ref{fig:task} (c), even though it has two identical loops, it does not contain any repetition subsequence since the loops are not adjacent.
    
    \begin{figure}[t]
        \centering
        \includegraphics[width=0.65\columnwidth]{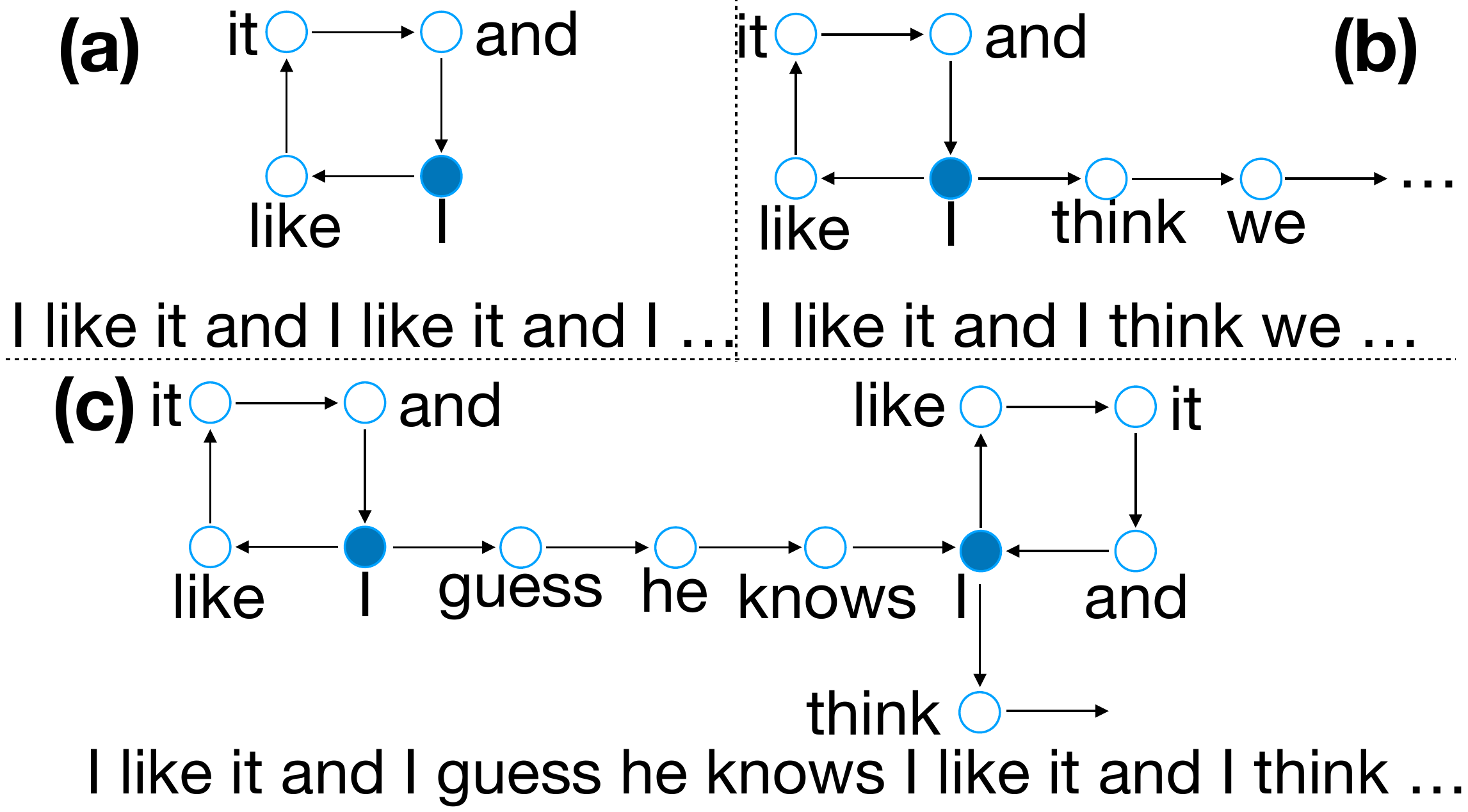}
        \caption{Loop and repetition subsequence.}
        \label{fig:task}
      \end{figure}
    
    \begin{definition}
      [General Generation Model] A general generation model is defined as $p_\theta(w_i|w_{i-1},\cdots,w_1,x)$, in which $\theta$ stands for all the model parameters; $w_i$ is the current word to be predicted; $w_{i-1},\cdots,w_1$ are the previous words that have been generated and $x$ is the input information.
    \end{definition}
    The general generation model is an abstract description of nearly all text generation models. For language models~\cite{radfordimproving,radfordlanguage,brown2020language}, the input $x$ is empty while $w_i$ is the word generated at the $i$th step. For translation models~\cite{sutskever2014sequence,bahdanau2014neural,luong2015effective}, $x$ stands for the source language. For summarization models, $x$ is the original text while $w_i$ is the $i$th word in the corresponding summary. However, analyzing the general generation model is almost impossible due to its complicated structure with many high non-linear layers. To facilitate effective analysis, we propose to first analyze a simplified model and study the relationship between them.

    \begin{definition}
      [Markov Generation Model] A Markov generation model predicts a word only based on the previous word, which can be denoted as $p_\xi(w_i|w_{i-1}) \approx p_\theta(w_i|w_{i-1},\cdots,w_1,x)$, where $\xi$ stands for parameters of the Markov generation model.
    \end{definition}
    The Markov generation model simplifies the general generation model with the intuition that the current word is often most influenced by the previous word. We conduct our analysis on the Markov generation model first and then extend to the general generation model. Since the probability of the current word only depends on the previous word, it becomes a Markov chain. We denote the transition matrix as $A\in \mathbb{R}^{(n+1)\times (n+1)}$, in which $n$ is the vocabulary size. $A$ is a stochastic matrix and it is guaranteed that $A_{ij}\ge 0,\sum_{j=1}^{n+1}A_{ij}=1,\forall i\in[1,n+1]$. We denote the $(n+1)$th word as the End-Of-Sentence (EOS) tag and the generation stops if it generates the EOS tag. Therefore, the EOS state is an absorbing state and $A$ can be written as $A=\begin{bmatrix}B,b\\0,1 \end{bmatrix}$, in which $B\in \mathbb{R}^{n\times n},b\in \mathbb{R}^{n\times 1}$. $B$ is a sub-transition matrix without the EOS tag and $B_{ij}\ge 0,\sum_{j=1}^n B_{ij}\le 1,\forall i,j\in [1,n]$.

    In order to derive a quantitative definition of the repetition problem, we first study loops in the Markov generation model. It is easy to observe that the matrix $B$ is very sparse since most words only have a limited number of subsequent words. We define the sparsity of matrix $B$ as $\zeta=\frac{1}{n^2}\sum_i^n\sum_j^n \mathds{1}(B_{ij}>0)$, in which $\mathds{1}(C)$ is the indicator function. It equals to 1 if $C$ is true and equals to 0 otherwise. Therefore, one word can transit to $\zeta n$ words on average at a single step. As illustrated in Fig. \ref{fig:h}, for a word to transit back to itself at the $k$th step, it generates $(\zeta n)^k$ paths and has $\frac{(\zeta n)^k}{n}$ paths that transit back to itself on average.
    On the other hand, the probability for the $i$th word loops back to itself after $k$ steps is the $i$th diagonal value of $B^k$ which is denoted as $B^k_{ii}$. The average probability for each path is $\frac{nB^k_{ii}}{(\zeta n)^{k}}$ and the average probability for a loop to repeat again is $(\frac{nB^k_{ii}}{(\zeta n)^k})^2$. The average probability for all loops to repeat again is $(\frac{nB^k_{ii}}{(\zeta n)^k})^2 \cdot \frac{(\zeta n)^k}{n}$. Since $(B^k_{ii})^2 \le (B^{2k})_{ii}$, to simplify the analysis, we consider the upper bound for all states and the probability sum as $\operatorname{tr}(\frac{nB^{2k}}{(\zeta n)^k})$, in which $\operatorname{tr}$ stands for the matrix trace. The average probability sum for all steps is $\sum_{k=1}^\infty \operatorname{tr}(\frac{B^{2k}}{\zeta^k n^{k-1}})$. The average probability for each word is $R=\frac{1}{n}\sum_{k=1}^\infty \operatorname{tr}(\frac{B^{2k}}{\zeta^k n^{k-1}})=\sum_{k=1}^\infty \operatorname{tr}(\frac{B^{2k}}{\zeta^k n^{k}})$. Therefore, we define the average repetition probability as follows:
    
    \begin{figure}[t]
        \centering
        \includegraphics[width=0.5\columnwidth]{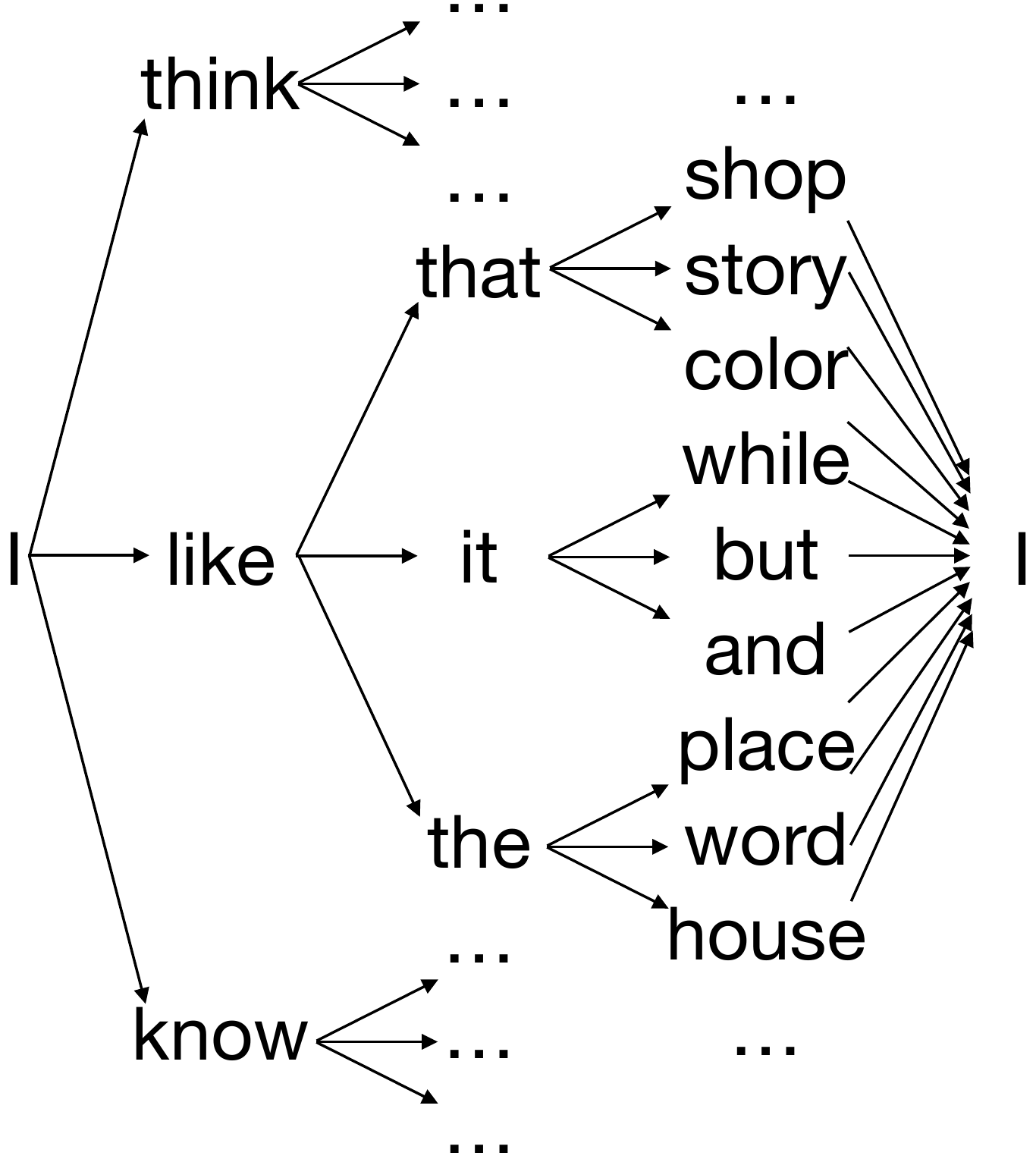}
        \caption{Illustration of loops starting and ending with the same word ``I''. On average, each word transits to three subsequent words.}
        \label{fig:h}
      \end{figure}
    
    \begin{definition}
      [Average Repetition Probability] Given a Markov generation model with sub-transition matrix $B\in \mathbb{R}^{n\times n}$ and non-zero probability $\zeta$  as defined above, the Average Repetition Probability (ARP) is defined as: 
      $$R=\sum_{k=1}^\infty \operatorname{tr}(\frac{B^{2k}}{\zeta^k n^{k}})$$
    \end{definition}
    The Average Repetition Probability (ARP) gives a quantitative description of the repetition problem. Intuitively, if ARP is high, the repetition problem occurs frequently and if ARP is small, the repetition problem seldom happens. 
    
    To better understand how existing methods work, we show that most of the existing out-of-the-box methods are equivalent to applying a transformation $\mathcal{T}$ to the word probability distribution $p$ and then applying a stochastic sampling on the transformed probability for obtaining the word $w\sim \mathcal{T}(p)$. For the \textbf{Stochastic} sampling \cite{fan2018hierarchical,holtzman2020curious}, $\mathcal{T}_s(p)=p$ is the identity operator. The transformation for the \textbf{Greedy} sampling \cite{opennmt,ott2019fairseq} changes the probability of the most probable word to 1 and sets others' probabilities to 0 which can be denoted as $\mathcal{T}_g(p)_i=\mathds{1}(i=\arg\max_j p_j)$. The \textbf{Topk} sampling \cite{fan2018hierarchical} method sets all word probabilities to 0 except top $k$ words and rescales it to a new distribution. It can be denoted as $\mathcal{T}_k(p)_i=\mathds{1}(i\in K) p_i / \sum_{j\in K} p_j$, in which
    $K$ is the Topk probability word set: $K=\arg\max_K \sum_{j\in K} p_j,s.t. |K|=k $. The \textbf{Nucleus} sampling \cite{holtzman2020curious} sets all word probabilities to 0 except for words with probabilities sum to larger than $p$. It can be denoted as $\mathcal{T}_n(p)_i=\mathds{1}(i\in N) p_i / \sum_{j\in N} p_j$, in which
    $N$ is the smallest word set with probability sum larger than $p$: $N=\arg \min_N |N|,s.t. \sum_{j\in N}p_j \ge p $. The Temperature sampling  (\textbf{Temp}) \cite{ackley1985learning,ficler2017controlling,caccia2019language} rescales the probability with a temperature parameter $t$ and can be denoted as $\mathcal{T}_t(p)_i= \exp((\log p_i)/t) / \sum_{j} \exp((\log p_j)/t)$. The Length Penalty (\textbf{LP}) \cite{wu2016google,opennmt} method prefers shorter sentences than long sentences. Without loss of generality, we simply enlarge the probability of the EOS tag with a constant $\beta$ which can be denoted as $\mathcal{T}_l(p)_i=\exp(\tilde{\ell}_i)/\sum_{j}\exp(\tilde{\ell}_j)$, in which $\tilde{\ell}_i=\ell_i + \mathds{1}(i=n+1)\beta$ and $\ell$ is the logits vector calculated by the model.

    It can be directly concluded from the definition that the Greedy sampling always has serious repetition problem. In greedy sampling, each word only takes a fixed subsequent word and thus $\zeta n = 1$. Therefore, ARP can be very large and even diverges to infinity which is consistent with the observation that Greedy sampling always causes repetition sentences in many generation systems~\cite{fan2018hierarchical,holtzman2020curious}.
    
    \subsection{Average Repetition Probability Upper Bounds}
    To provide more intuitive understanding of the repetition problem directly, we simplify ARP and derive several upper bounds for further investigation.
    \begin{theorem}
      \label{thm:bound}
      If $\zeta n > \rho (B^2)$,
      $$R=\operatorname{tr}(B^2(\zeta n I-B^2)^{-1})\le \frac{\|B^2\|_*}{\sigma_n(\zeta nI-B^2)},$$
    in which $\|\cdot\|_*$ is the nuclear norm; $\sigma_n$ denotes the smallest singular value; $\rho(\cdot)$ is the spectral radius of the matrix; $I$ is the identity matrix.
    \end{theorem}

    The proof is provided in Appendix \aref{\ref{sec:proof-thm-bound}}{A.1} \cite{fu2020a}. Theorem \ref{thm:bound} gives a very straightforward upper bound of the ARP. It can be easily proved that $\rho (B^2)<1$ for a sub-stochastic matrix $B^2$. Therefore, Theorem \ref{thm:bound} fails only if $\zeta n \le \rho (B^2) < 1$, which means that all words have less than one subsequent word. It seldom happens in a normal language model. As neither the nuclear norm nor the smallest singular value has very intuitive physical meaning in generation models, we derive two corollaries based on Theorem \ref{thm:bound} to get more practical and intuitive upper bounds.

    \begin{corollary}
        \label{col:sample}
    $$R \le \frac{\sqrt{r}(\sum_{i=1}^n \sum_{j=1}^n (B_{ij}-\mu_i)^2 +\sum_{i=1}^n (1-b_i)^2)}{\sigma_n(\zeta nI-B^2)},$$
    where $r$ is the rank of $B^2$ and $\mu_i=\frac{\sum_{k=1}^n B_{ik}}{n}$ is the mean of each row of $B$.
    \end{corollary}

    The proof is provided in Appendix \aref{\ref{sec:proof-col-sample}}{A.2} \cite{fu2020a}. It can be concluded from Corollary \ref{col:sample} that the upper bound of $R$ decreases as the variance of $B_{ij}$ decreases. The current methods including Nucleus sampling, Top-k sampling, and stochastic sampling are of this kind. Compared with the commonly used greedy sampling method in which only one word has a probability of 1 with others being 0, these methods significantly decrease the variance and thus lower the upper bounds. 
    It should be noted that Stochastic sampling always achieves the smallest variance. However, due to the long tail effect, it cannot be used since it has a very high probability of sampling low probability words. 
    It is also obvious that Temperature sampling alleviates the repetition problem because it controls the variance by changing the temperature parameter. 
    It is interesting to observe that directly truncating words with low probabilities actually increases the variance since it sets a lot of words' probabilities to 0 while increasing the probabilities of the remaining words. Therefore, Topk/Nucleus sampling alone cannot alleviate the repetition problem. They should always work together with the Stochastic sampling method or the Temperature sampling method, and the real role of Topk/Nucleus sampling is to solve the long-tail problem of stochastic sampling or Temperature sampling. If the temperature is fixed, using Topk/Nucleus sampling can even make the repetition problem worse.
    On the other hand, the length penalty strategy also alleviates the repetition problem. The reason is that it is equivalent to increasing the probability of going to the EOS state which increases the elements in $b$. From Corollary \ref{col:sample}, when $|b|$ increases, $R$ decreases. We will conduct experiments to validate all the above claims.

    \begin{corollary}
        \label{col:inout}
    If $\zeta n I-B^2$ is a diagonally dominant matrix,
    $$\begin{aligned}
    R\le\frac{\|B^2\|_*}{\min_{1\le i \le n} \{ \frac{1}{2}( \zeta n - \underbrace{\sum_{j=1}^n (B^2)_{ij}}_{outflow})+\frac{1}{2}(\zeta n - \underbrace{\sum_{k=1}^n (B^2)_{ki}}_{inflow}) \}}.
    \end{aligned}$$
    \end{corollary}

    The proof is provided in Appendix \aref{\ref{sec:proof-col-inout}}{A.3} \cite{fu2020a}. In Corollary \ref{col:inout}, the inflow for a word is the probability sum of all words that take it as the subsequent word. If it is too big, the upper bound can be magnified extensively and fails to limit the ARP. This observation theoretically justifies the claim that high inflow words are more likely to go back to itself and cause the repetition problem. We can also conclude that increasing $|b|$ can decrease the outflow term and thus decrease $R$. It also shows that the length penalty method can help alleviate the repetition problem. As is shown in many previous works, most repetition sentences contain high-frequency words. By Corollary \ref{col:inout}, we can conclude that it is not the high-frequency words, but the high inflow words that really lead to repetition generation. We propose our novel rebalanced encoding based on this corollary to reduce the high inflow words.
    
    \subsection{Extension to General Generation Model}
    
    To extend our analysis to the general generation model, we define the ARP for the general generation model and show that if the transition matrix is controlled, the ARP will not deviate a lot.
    
    \begin{definition}
      [Average Repetition Probability for General Generation Model] For a general generation model, the sub-transition matrix for the $k$th step can be expressed as $B'_k=B+T_k$, in which $T_k\in \mathbb{R}^{n\times n}$ is a perturbation matrix and each element for ${T_k}_{(ij)}$ is independently distributed with mean 0 and we assume that the variance is controlled as $\delta^2 < \frac{1}{n}$. The general average repetition probability is defined as:
      $$R'=\sum_{r=1}^\infty \operatorname{tr}(\frac{\prod_{k=1}^{2r}B'_k}{\zeta^r n^{r}}).
      $$
      \label{def:garp}
    \end{definition}
    
    In the general generation model, each step has its own sub-transition matrix $B'_k$. Since the current word is mostly affected by the previous word, we denote $B'_k$ as applying a ``small'' perturbation matrix $T_k$ on $B$. We show that if $T_k$ is ``small'', $R'$ will not deviate much from $R$.

    \begin{theorem}
        \label{thm:extend}
        For a general generation model, if $\sum_{i=1}^n B_{ij}^2<1,\zeta n>4$, then for every constant $a>0$ we have $$\operatorname{Pr}(|R-R'|\ge a)\le \frac{3\zeta n\delta^2 }{a^2(\zeta n-4)(\zeta n-1)}.$$
    \end{theorem}

    The proof is provided in Appendix \aref{\ref{sec:proof-thm-extend}}{A.4} \cite{fu2020a}. Theorem \ref{thm:extend} shows that given a generation model that has a different transition matrix $B'_k$ at each step, if the transition matrix does not deviate a lot from that in the Markov generation model, the deviation of ARP is bounded with high probability. It should be noted that the distribution of ${T_k}_{(ij)}$ is influenced by $B$. Nevertheless, Theorem \ref{thm:extend} only depends on the mean and the variance regardless of the specific value constraint.

    \section{Rebalanced Encoding}
    It can be concluded from Corollary \ref{col:inout} that the ARP upper bound decreases as the inflow term decreases. It can also be inferred from Corollary \ref{col:sample} that if the variance of $B$ decreases, the ARP upper bound will also decrease. Inspired by these two properties, we propose a novel rebalanced encoding (RE) which merges the high inflow pairs into single words and thus reduces the inflow term and the variance of $B$ simultaneously. Different from most of the existing methods that focus on the post-editing of the probability distribution, the RE method uses a novel word encoding scheme and thus makes the model predict a good probability distribution instead of post-editing. 
    
    Our proposed RE method first applies traditional BPE encoding \cite{sennrich2016neural} to the training text. Then, it makes a statistical transition matrix with the encoded training text. It picks high inflow pairs that have transition probability higher than a threshold $\gamma$ and merges the word pair as a whole word. Specifically, if two words are split from BPE as subwords, we simply merge them and remove the BPE tags. For example, for the high inflow word pair ``(de@@, crease)'', we replace all ``de@@ crease'' to ``decrease''. If the two words are words that have not been split by BPE, we merge them by adding a ``=='' tag between them and treat them as a single new word. For example, for the high inflow word pair ``(involved, in)'', we replace all ``involved in'' to ``involved==in''. We rebuild the transition matrix and repeat the above procedure until all the probabilities in the transition matrix are less than $\gamma$ or we reach a specific iteration epoch. The detailed algorithm is shown in Algorithm \ref{algo:re} and we follow the presentation of using Python code \cite{sennrich2016neural} to make it more concise.
    
    \begin{algorithm}[t]
      \begin{lstlisting}[language=Python]
def learnRE(words : list, N : int, gamma : float):
  merges = []
  for step in range(N):
    id_to_word = list(set(words))
    word_to_id = {w:i for i,w in enumerate(id_to_word)}
    M = numpy.zeros([len(id_to_word), len(id_to_word)])
    for i in range(len(words) - 1):
      M[word_to_id[words[i]], word_to_id[words[i+1]]]+=1
    M =  M / M.sum(1).reshape(-1,1).clip(1)
    if M.max() <= gamma: break
    merges += [(id_to_word[i1], id_to_word[i2]) for  
              i1, i2 in zip(*(M > gamma).nonzero())]
    words = applyRE(words, merges)
  return merges
    
def applyRE(words : list, merges : list):
  for merge in merges:
    for i in range(len(words) - 1):
      if tuple(words[i : i + len(merge)]) == merge:
        words[i : i + len(merge)] = [
            "==".join(merge).replace("@@==", "")]
        i -= 1
  return words
          
                  
      \end{lstlisting}
      \caption{Rebalanced Encoding Algorithm
      }
      \label{algo:re}
      \end{algorithm}

    \section{Experiments}

    \subsection{Experimental Settings}
    We conduct our experiments on two common text generation tasks, namely, the neural machine translation task and the language modeling task.
    
    \textbf{Neural Machine Translation (NMT).} The NMT task translates sentences from the source domain to the target domain with the most prevalent Transformer \cite{vaswani2017attention,ott2019fairseq} architecture. We adopt the widely used IWSLT’14 English-German dataset containing 160K sentences pairs. We encode all the sentences with the byte-pair encoding \cite{sennrich2016neural} with subword units of 10,000 for each language. The overall translation performance is evaluated by BLEU \cite{papineni2002bleu} and ROUGE$_L$ \cite{lin2004rouge}. The repetition problem is quantitatively evaluated with rep-w, rep-n \cite{welleck2020neural} and rep-r. rep-w measures the repetition words proportion by the fraction of the current token that occurs in the previous $w$ tokens. It can be calculated as $\operatorname{rep-w}=\frac{1}{|\mathcal{D}|}\sum_{s\in \mathcal{D}}\frac{1}{|s|}\sum_{t=1}^{|s|}\mathds{1}[s_t\in s_{t-w-1:t-1}]$, in which $s$ denotes generated sentences in the result set $\mathcal{D}$ and $|s|$ denotes the word count of $s$; $s_t$ is the $t$th word in $s$ while $s_{t-w-1:t-1}$ is the sub-sequence of $s$ from the $(t-w-1)$th word to the $(t-1)$th word. The rep-w metric extends the rep/$\ell$ in \citet{welleck2020neural} in which the sentence length is fixed. rep-n is defined as $\operatorname{rep-n} =1.0-\frac{|\{\tilde{s}|\exists p \in [1, |s| - n+1], \tilde{s}=s_{p:p+n-1}\}|}{|s|-n+1}$. rep-n depicts the repetition problem for the specific gram $n$ while rep-w characterizes the repetition problem in the view of words. To give a quantitative metric for the repetition problem by length and avoid setting specific parameters, we propose the rep-r metric. rep-r stands for the ratio of the repetition snippet in a sentence measured by length. It is defined as $\operatorname{rep-r}=\frac{|\{i|(s_i=s_j\land s_{i+1}=s_{j+1}, \exists j\ne i) \lor (s_i=s_k\land s_{i-1}=s_{k-1}, \exists k\ne i) \}|}{|s|}$, in which $\land$ and $\lor$ are logical ``and'' and ``or'' respectively. The details of hyper-parameters settings are presented in Appendix \aref{\ref{sec:appendix-hyperpars}}{A.6} \cite{fu2020a}.
    
    \begin{table}[t]
      \centering
      \small
      
      \begin{tabular}{@{~}l@{~}@{~}l@{~}@{~}l@{~}@{~}l@{~}@{~}l@{~}@{~}l@{~}}
        \toprule
        {Method} &    rep-w$\downarrow$ & rep-n$\downarrow$ &    rep-r$\downarrow$ &    BLEU$\uparrow$ & ROUGE$_L$ $\uparrow$ \\
        \midrule
        Greedy       &  0.0883 &    0.0330 &  0.0512 &  0.352 &   0.606 \\
        Stochastic   &  0.0783 &    0.0272 &  0.0337 &  0.222 &   0.472 \\
        Temp ($t$=0.15)  &  0.0879 &    0.0328 &  0.0511 &  0.351 &   0.605 \\
        Topk ($k$=10)      &  0.0882 &    0.0329 &  0.0507 &  0.350 &   0.605 \\
        Topk ($k$=40)      &  0.0881 &    0.0329 &  0.0511 &  0.350 &   0.604 \\
        Nucleus ($p$=0.9)  &  0.0878 &    0.0328 &  0.0508 &  0.349 &   0.603 \\
        Nucleus ($p$=0.95) &  0.0882 &    0.0329 &  0.0510 &  0.350 &   0.604 \\
        LP ($\beta$=6)         &  0.0863 &    0.0322 &  0.0500 &  0.349 &   0.605 \\
        \hline
        RE ($\gamma$=0.15)      &  0.0768 &    0.0281 &  0.0419 &  0.350 &   0.608 \\
        RE ($\gamma$=0.1)       &  0.0743 &    0.0275 &  0.0417 &  0.350 &   0.607 \\
        RE ($\gamma$=0.05)      &  0.0585 &    0.0211 &  0.0296 &  0.340 &   0.603 \\
        RE ($\gamma$=0.02)      &  0.0434 &    0.0158 &  0.0221 &  0.335 &   0.600 \\
        \bottomrule
        \end{tabular}

      \caption{Experimental results for NMT task.\footnotemark }
      
      \label{tab:nmt}
      \end{table}

    \begin{table}[t]
      \centering
      \small

      \begin{tabular}{lllll}
        \toprule
        {Method} &   rep-w$\downarrow$ & rep-n$\downarrow$ &   rep-r$\downarrow$ &   ppl-c \\
        \midrule
        Greedy       &  0.590 &    0.733 &  0.917 &  0.150 \\
        Stochastic   &  0.120 &    0.092 &  0.155 &  7.320 \\
        Temp ($t$=0.75)  &  0.254 &    0.215 &  0.409 &  1.060 \\
        Topk ($k$=40)      &  0.235 &    0.188 &  0.363 &  0.969 \\
        Topk ($k$=10)      &  0.251 &    0.195 &  0.348 &  0.962 \\
        Nucleus ($p$=0.9)  &  0.234 &    0.195 &  0.368 &  1.020 \\
        Nucleus ($p$=0.95) &  0.227 &    0.191 &  0.322 &  1.090 \\
        LP ($\beta$=7)         &  0.547 &    0.660 &  0.829 &  1.050 \\
        \hline
        RE ($\gamma$=0.1)       &  0.196 &    0.180 &  0.321 &  0.974 \\
        RE ($\gamma$=0.08)      &  0.180 &    0.156 &  0.286 &  1.010 \\
        \bottomrule
        \end{tabular}

    \caption{Experimental results for LM task.\footnotemark[\value{footnote}]}
    \label{tab:lm}
    \end{table}
    
    \footnotetext{$\uparrow$ means larger is better while $\downarrow$ means lower is better.}

    \begin{figure*}[t]
      \centering
      \begin{minipage}[t]{0.42\textwidth}
      \centering
      \includegraphics[width=0.9\columnwidth]{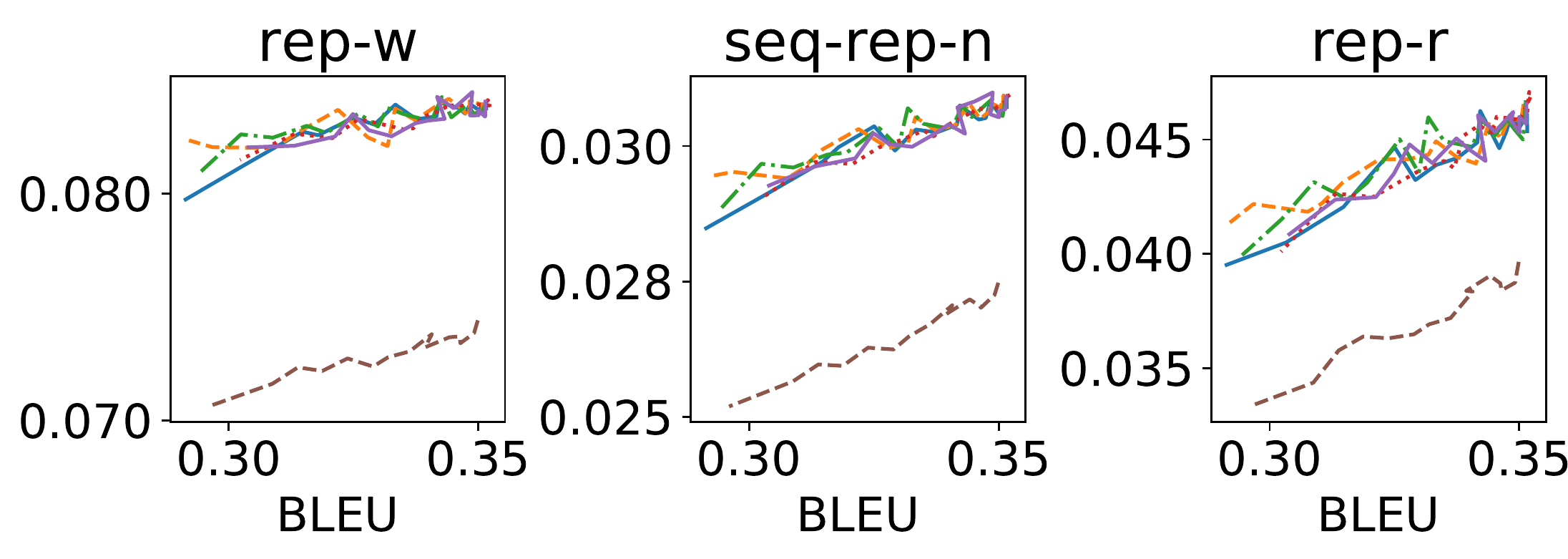}
      \caption*{(a) NMT}
      \label{fig:bleu-rep}
      \end{minipage}
      \quad 
      \begin{minipage}[t]{0.55\textwidth}
      \centering
      \includegraphics[width=0.92\columnwidth]{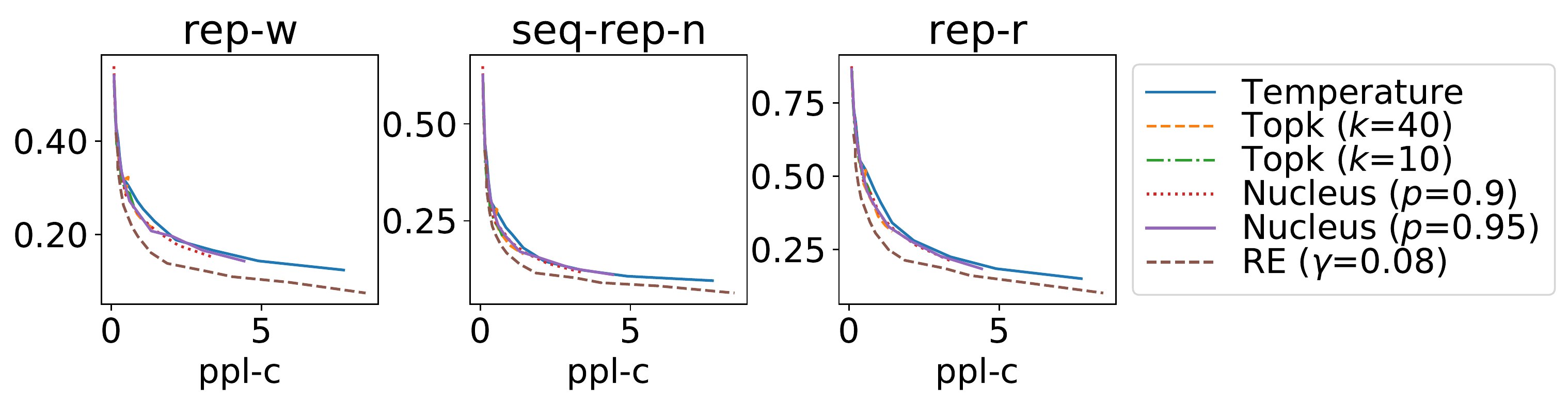}
      \caption*{(b) LM}
      \label{fig:ppl-scores}
      \end{minipage}
      \caption{Repetition performance  balancing.}
      \label{fig:rpb}
    \end{figure*}

    \begin{figure*}[t]
      \begin{minipage}{0.31\textwidth}
      \centering
      \includegraphics[width=0.82\columnwidth]{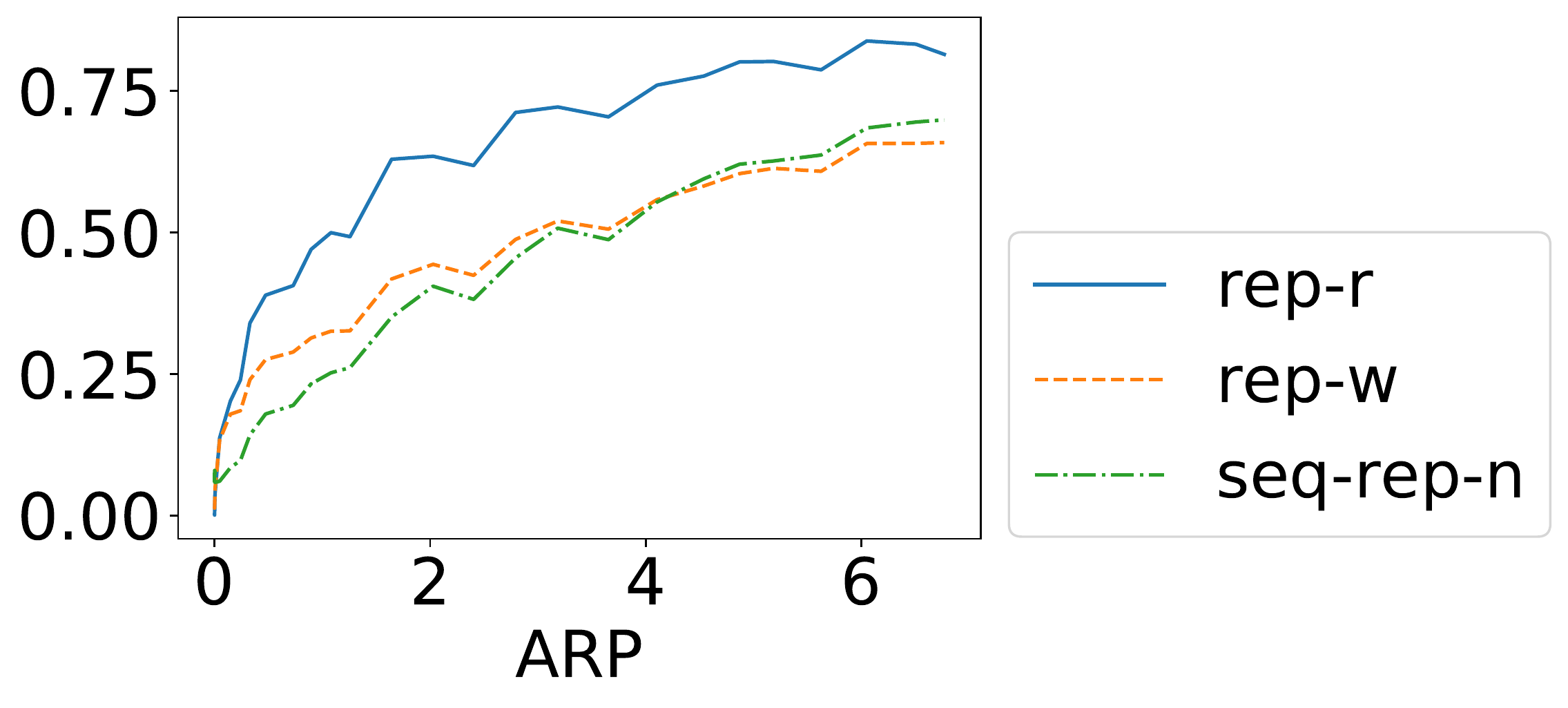}
      \caption{Correlation of ARP and repetition metrics.}
      \label{fig:arp-correlation}
      \end{minipage}
      \quad 
      \begin{minipage}{0.31\textwidth}
      \centering
      \includegraphics[width=0.9\columnwidth]{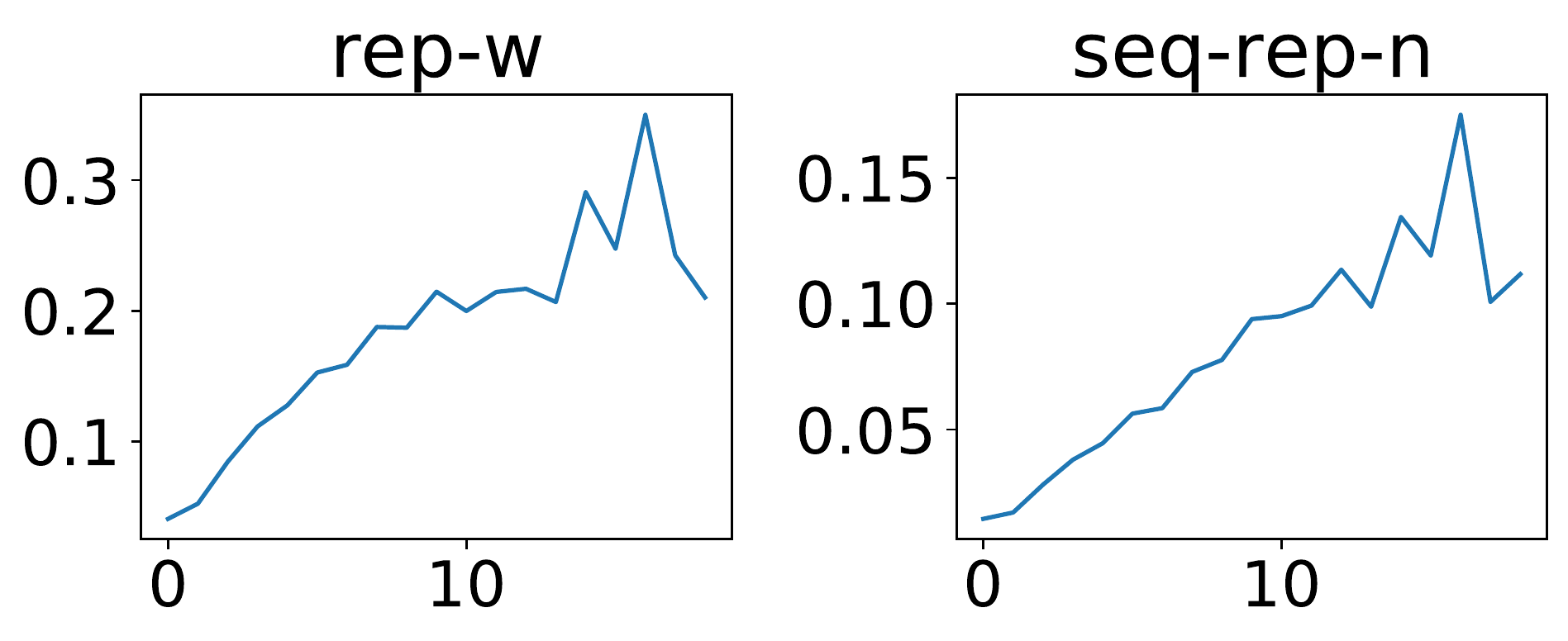}
      \caption{Influence of high inflow pair. $x$ axis is the high inflow pair count.}
      \label{fig:reason-high-inflow}
      \end{minipage}
      \quad 
      \begin{minipage}{0.3\textwidth}
          \centering
          \small
          \begin{tabular}{@{~}l@{~}@{~}l@{~}@{~}l@{~}}
            \toprule
            {Method} & Overall$\uparrow$ & Repetition$\uparrow$ \\
            \midrule
            Greedy      &   3.642 &      3.342 \\
            Nucleus ($p$=0.9) &    3.650 &      3.383 \\
            RE ($\gamma$=0.08)      &    3.800 &        3.400 \\
            \bottomrule
            \end{tabular}
          \captionof{table}{Human Evaluation.\footnotemark[\value{footnote}]}
          \label{tab:humaneval}
        \end{minipage}
    \end{figure*}

    \textbf{Language Modeling (LM).} LM task is a fundamental task for natural language processing which predicts the probability of a word conditioned on previous words. We adopt the language model with Transformer decoder which has also been utilized by the most prevalent  GPT models~\cite{radfordimproving,radfordlanguage,brown2020language}. We use the Wiki-103 dataset~\cite{merity2017pointer} and encode the text with byte pair encoding with subword units around 10,000. We use the same language model as \citet{ott2019fairseq} and keep all hyper-parameters the same. Since the perplexity score strongly depends on the encoding approaches we use for making subword units, it is not directly comparable when different encoding approaches are applied. We propose a new metric called perplexity coefficient (ppl-c) which divides the perplexity of the generated text with the average perplexity of the text in the test set. Therefore, a sentence is more likely to be written by humans if its ppl-c score closes to 1.0. To evaluate the model, we take words in the dictionary to feed them into the language model and generate subsequent sequences. We guarantee that the initial words for all models are the same. The repetition problem is evaluated with rep-w, rep-n \cite{welleck2020neural}, and rep-r. The Topk, Nucleus, and RE methods are combined with the Temperature sampling method together. To make the models comparable with each other, we tune all the models to make the ppl-c scores as close to 1.0 as possible which is closest to human-written text. The details of hyper-parameters settings are presented in Appendix \aref{\ref{sec:appendix-hyperpars}}{A.6} \cite{fu2020a}.
    
    \textbf{Comparison Methods.}
    As our proposed RE method does not need any additional modification of the generation model. we compare it with the most widely used out-of-the-box methods including Greedy sampling, Stochastic sampling, Temperature sampling, Topk sampling, Nucleus sampling, and Length Penalty (LP) method.
    The details of each method have been discussed in Section \ref{sec:def}.

    \begin{table*}[t]
      \centering
      \small
      
      \begin{tabular}{@{~}L{1cm}@{~}@{~}L{15cm}@{~}}
      \toprule
      Method &    \multicolumn{1}{c}{Generated Text} \\
      \midrule
      Greedy & Battalion , the 1st Battalion , the 2nd Battalion , the 1st Battalion , the 1st Battalion , the 1st Battalion , the 1st Battalion , the 1st Battalion , the 1st Battalion , the 2nd Battalion , the 1st Battalion , the 1st Battalion , the 1st Battalion , the 2nd Battalion , the 1st Battalion , the 1st Battalion , the 1st Battalion , the 1st Battalion ... \\
      \hline
      Topk ($k$=10) & Battalion of the Royal Marines were also assigned to the 1st Division , the 1st Division , and the 1st Division . The 1st Division was assigned to the 1st Division , and the 1st Division was assigned to the 1st Division . The 1st Division was assigned to the 1st Division , and the 1st Division was assigned to the 1st Division . The 1st ...\\ 
      \hline
      RE & Battalion was the first unit to be deployed to Iraq in March 1971 . In April 1971=, the battalion was involved=in Operation Ira@@ q@@ i Fre@@ edom , a major operation in=the Ira@@ q@@ i conflict . In January 1971=, the battalion was involved=in Operation Ira@@ q@@ i Fre@@ edom , a major operation in=the Ira@@ q@@ i Gulf . In March 1971=, the battalion was involved=in Operation Ira@@ q@@ i Fre@@ edom , a major operation in=the Ira@@ q@@ i Gulf . In April 1971=, the battalion was involved=in Operation Ira@@ q@@ i Fre@@ edom , a major operation in=the Ira@@ q@@ i Gulf . $\langle$eos$\rangle$ \\
      \bottomrule
      \end{tabular}

    \caption{Experimental results for LM task.}
    \label{tab:casestudy}
    \end{table*}
    
    \subsection{Experimental Results}
    
    \textbf{Main Results.} The NMT experimental results are shown in Table \ref{tab:nmt} while the LM Experimental results are shown in Table \ref{tab:lm}. We have the following conclusions from the results. (1) The RE method alleviates the repetition problem and outperforms existing methods significantly. (2) The repetition problem of the Greedy method is more serious than other models while it hardly appears in the Stochastic sampling method. This result justifies our theoretical analysis in Corollary \ref{col:sample} that higher variance in the transition matrix leads to a higher probability of repetition. (3) The LP method and RE method both alleviate the repetition problem. The reason is that the RE method controls the inflow term while the LP method controls the outflow term and thus they limit the upper bound of ARP. This result justifies our theoretical analysis in Corollary \ref{col:inout}. (4) Empirically, the performance metrics become slightly worse (BLEU and ROUGE$_L$ decrease or ppl-c increases when larger than 1.0) as the repetition problem is alleviated. This is because existing performance metrics are not very sufficiently sensitive to the repetition problem and they will not become significantly better even when the repetition problem is well alleviated. For example, the ppl-c metric will not become worse even if it adds many repetition sequences with high prediction probability. Therefore, the existing metrics are not enough to provide a comprehensive evaluation and thus we conduct a human evaluation.

    \noindent \textbf{Human Evaluation.} We conduct a human evaluation and the results are shown in Table \ref{tab:humaneval}. We sample 120 generated text in the LM task for each model. The sampled text is scored by human helpers to evaluate the overall performance and repetition performance. The detailed questionnaire is shown in Appendix \aref{\ref{sec:appendix-questionnaire}}{A.5} \cite{fu2020a}. It can be concluded from the results that our proposed method outperforms the comparison methods in both alleviating the repetition problem and improving the overall score simultaneously.

    \noindent \textbf{Repetition Performance Balancing.} As is shown in our main results, mitigating the repetition problem and improving the evaluation metrics are not simultaneously feasible since existing performance metrics are not sensitive to repetition. In order to show that our proposed RE method achieves the best balance between the two goals, we conduct a repetition performance balancing experiment. We tune the temperature parameter for the models to get different repetition metrics at different levels of performance (BLEU or ppl-c). The NMT task results are illustrated in Fig. \ref{fig:rpb} (a) while the LM task results are shown in Fig. \ref{fig:rpb} (b). It can be concluded from the results that (1) The RE method achieves the lowest repetition scores at the same performance score in both NMT and LM tasks. (2) The Topk sampling and Nucleus sampling alleviate the repetition problem of the Temperature sampling method because they help alleviate the long-tail effect and thus improve the overall performance. This experiment justifies the claims in Corollary \ref{col:sample}.

    \noindent \textbf{Well-definedness of ARP.} To show that our proposed definition of ARP is well defined, we conduct an experiment to study the relationship between the theoretical ARP and the repetition metrics. To calculate the theoretical ARP, we conduct this experiment on a Markov generation model. The Markov transition matrix is calculated by counting words in Wiki-103. We sample each word with the Temperature sampling method and we change the temperature parameter to study the relation of ARP and repetition metrics. The results are shown in Fig. \ref{fig:arp-correlation}. It can be concluded from the results that as the ARP grows, all repetition metrics are increasing. This positive correlation shows that ARP is well-defined.

    \noindent \textbf{Influence of High Inflow Pairs.} To show that the high inflow pairs do cause the repetition problem, we conduct an experiment in the NMT task to show the relationship between the high inflow pair count and the repetition metrics. As shown in Fig. \ref{fig:reason-high-inflow}, we calculate the high inflow pair count and the repetition scores in each sentence. It can be concluded from the results that if one sentence contains too many high inflow pairs, it is more likely to get a higher repetition score. This result also justifies Corollary \ref{col:inout}.

    \noindent \textbf{Case Study.} To get a better understanding how the repetition is alleviated, we show a case study by sampling results from different methods. The results are shown in Table \ref{tab:casestudy}. It can be observed from the results that the Greedy sampling method has a severe repetition problem. The Topk sampling method alleviates the repetition problem by minimizing ppl-c and can thus increase the temperature when the ppl-c level is fixed. Our proposed RE method gives the best trade-off between the repetition problem and the overall performance.

    \section{Related Works}
    Text generation tasks aim at generating sentences based on given input. It can be divided into two main categories, namely, the sequence-to-sequence based systems \cite{sutskever2014sequence,bahdanau2014neural,luong2015effective,wu2016google,vaswani2017attention} and the language model based systems \cite{radfordimproving,radfordlanguage,brown2020language}. The sequence-to-sequence framework has been applied to solve many tasks including neural machine translation \cite{sutskever2014sequence,bahdanau2014neural,luong2015effective}, summarization \cite{nallapati2016abstractive} and data-to-text generation \cite{lebret2016neural,wiseman2017challenges,wiseman2018learning,novikova2017e2e,fu2020open,fu2020partially,fu2020dynamic}. On the other hand, language model based systems are applied in the open-ended generation tasks. For example, OpenAI proposes the GPT models \cite{radfordimproving,radfordlanguage,brown2020language} which are based on language models.
    
    
    Many empirical out-of-the-box methods without too much model modification  have been proposed to alleviate the repetition problem by designing proper sampling methods to replace the widely used greedy sampling method \cite{opennmt,ott2019fairseq}. The details of these methods have been discussed in Section \ref{sec:def}. Apart from the out-of-the-box methods, \citet{welleck2020neural} propose the unlikelihood training approach with new components to solve the problem of likelihood training. 
    
    \section{Conclusions}
    We propose a novel theoretical analysis for the repetition problem. Our theory shows that the repetition problem is caused by the language itself. Too many high inflow words in the human language make it easy to go back to themselves and inevitably increase the repetition probability. Guided by this theory, we show that most of the existing methods are minimizing the ARP upper bounds explicitly or implicitly. Furthermore, we propose a novel rebalanced encoding method to tackle the repetition problem by reducing the high inflow terms. Our experiments validate our proposed theory and demonstrate the effectiveness of rebalanced encoding.
    
    \bibliography{reference}

\begin{thebibliography}{32}
\providecommand{\natexlab}[1]{#1}
\providecommand{\url}[1]{\texttt{#1}}
\providecommand{\urlprefix}{URL }
\expandafter\ifx\csname urlstyle\endcsname\relax
  \providecommand{\doi}[1]{doi:\discretionary{}{}{}#1}\else
  \providecommand{\doi}{doi:\discretionary{}{}{}\begingroup
  \urlstyle{rm}\Url}\fi

\bibitem[{Ackley, Hinton, and Sejnowski(1985)}]{ackley1985learning}
Ackley, D.~H.; Hinton, G.~E.; and Sejnowski, T.~J. 1985.
\newblock A learning algorithm for Boltzmann machines.
\newblock \emph{Cognitive science} 9(1): 147--169.

\bibitem[{Bahdanau, Cho, and Bengio(2014)}]{bahdanau2014neural}
Bahdanau, D.; Cho, K.; and Bengio, Y. 2014.
\newblock Neural machine translation by jointly learning to align and
  translate.
\newblock \emph{arXiv preprint arXiv:1409.0473} .

\bibitem[{Brown et~al.(2020)Brown, Mann, Ryder, Subbiah, Kaplan, Dhariwal,
  Neelakantan, Shyam, Sastry, Askell et~al.}]{brown2020language}
Brown, T.~B.; Mann, B.; Ryder, N.; Subbiah, M.; Kaplan, J.; Dhariwal, P.;
  Neelakantan, A.; Shyam, P.; Sastry, G.; Askell, A.; et~al. 2020.
\newblock Language models are few-shot learners.
\newblock \emph{arXiv preprint arXiv:2005.14165} .

\bibitem[{Caccia et~al.(2019)Caccia, Caccia, Fedus, Larochelle, Pineau, and
  Charlin}]{caccia2019language}
Caccia, M.; Caccia, L.; Fedus, W.; Larochelle, H.; Pineau, J.; and Charlin, L.
  2019.
\newblock Language GANs Falling Short.
\newblock In \emph{International Conference on Learning Representations
  (ICLR)}.

\bibitem[{Choi(2018)}]{choi2018missing}
Choi, Y. 2018.
\newblock The Missing Representation in Neural Language Models.
\newblock In \emph{3rd Workshop on Representation Learning for NLP (RepL4NLP)}.

\bibitem[{Fan, Lewis, and Dauphin(2018)}]{fan2018hierarchical}
Fan, A.; Lewis, M.; and Dauphin, Y. 2018.
\newblock Hierarchical Neural Story Generation.
\newblock In \emph{Proceedings of the 56th Annual Meeting of the Association
  for Computational Linguistics (Volume 1: Long Papers)}, 889--898.

\bibitem[{Ficler and Goldberg(2017)}]{ficler2017controlling}
Ficler, J.; and Goldberg, Y. 2017.
\newblock Controlling Linguistic Style Aspects in Neural Language Generation.
\newblock In \emph{Proceedings of the Workshop on Stylistic Variation},
  94--104.

\bibitem[{Fu, Bing, and Lam(2020)}]{fu2020open}
Fu, Z.; Bing, L.; and Lam, W. 2020.
\newblock Open Domain Event Text Generation.
\newblock In \emph{Thirty-Fourth AAAI Conference on Artificial Intelligence},
  7748--7755.

\bibitem[{Fu et~al.(2020{\natexlab{a}})Fu, Bing, Lam, and
  Jameel}]{fu2020dynamic}
Fu, Z.; Bing, L.; Lam, W.; and Jameel, S. 2020{\natexlab{a}}.
\newblock Dynamic topic tracker for kb-to-text generation.
\newblock In \emph{Proceedings of the 28th International Conference on
  Computational Linguistics}, 2369--2380.

\bibitem[{Fu et~al.(2020{\natexlab{b}})Fu, Lam, So, and Shi}]{fu2020a}
Fu, Z.; Lam, W.; So, A. M.-C.; and Shi, B. 2020{\natexlab{b}}.
\newblock A Theoretical Analysis of the Repetition Problem in Text Generation.
\newblock \emph{arXiv preprint arXiv:2012.14660} .

\bibitem[{Fu et~al.(2020{\natexlab{c}})Fu, Shi, Lam, Bing, and
  Liu}]{fu2020partially}
Fu, Z.; Shi, B.; Lam, W.; Bing, L.; and Liu, Z. 2020{\natexlab{c}}.
\newblock Partially-Aligned Data-to-Text Generation with Distant Supervision.
\newblock In \emph{Proceedings of the 2020 Conference on Empirical Methods in
  Natural Language Processing (EMNLP)}, 9183--9193.

\bibitem[{Holtzman et~al.(2020)Holtzman, Buys, Du, Forbes, and
  Choi}]{holtzman2020curious}
Holtzman, A.; Buys, J.; Du, L.; Forbes, M.; and Choi, Y. 2020.
\newblock The Curious Case of Neural Text Degeneration.
\newblock In \emph{International Conference on Learning Representations
  (ICLR)}.

\bibitem[{Johnson(1989)}]{johnson1989gersgorin}
Johnson, C.~R. 1989.
\newblock A Gersgorin-type lower bound for the smallest singular value.
\newblock \emph{Linear Algebra and its Applications} 112: 1--7.

\bibitem[{Klein et~al.(2017)Klein, Kim, Deng, Senellart, and Rush}]{opennmt}
Klein, G.; Kim, Y.; Deng, Y.; Senellart, J.; and Rush, A. 2017.
\newblock {OpenNMT}: Open-Source Toolkit for Neural Machine Translation.
\newblock \emph{Proceedings of the 55th Annual Meeting of the Association for
  Computational Linguistics, System Demonstrations} 67--72.

\bibitem[{Lebret, Grangier, and Auli(2016)}]{lebret2016neural}
Lebret, R.; Grangier, D.; and Auli, M. 2016.
\newblock Neural Text Generation from Structured Data with Application to the
  Biography Domain.
\newblock In \emph{Proceedings of the Conference on Empirical Methods in
  Natural Language Processing (EMNLP)}, 1203--1213.

\bibitem[{Lin(2004)}]{lin2004rouge}
Lin, C.-Y. 2004.
\newblock Rouge: A package for automatic evaluation of summaries.
\newblock \emph{Workshop on Text Summarization Branches Out} .

\bibitem[{Luong, Pham, and Manning(2015)}]{luong2015effective}
Luong, T.; Pham, H.; and Manning, C.~D. 2015.
\newblock Effective Approaches to Attention-based Neural Machine Translation.
\newblock In \emph{Proceedings of the Conference on Empirical Methods in
  Natural Language Processing (EMNLP)}, 1412--1421.

\bibitem[{Merity et~al.(2017)Merity, Xiong, Bradbury, and
  Socher}]{merity2017pointer}
Merity, S.; Xiong, C.; Bradbury, J.; and Socher, R. 2017.
\newblock Pointer Sentinel Mixture Models.
\newblock In \emph{International Conference on Learning Representations
  (ICLR)}.

\bibitem[{Nallapati et~al.(2016)Nallapati, Zhou, dos Santos, Guul{\c{c}}ehre,
  and Xiang}]{nallapati2016abstractive}
Nallapati, R.; Zhou, B.; dos Santos, C.; Guul{\c{c}}ehre, {\c{C}}.; and Xiang,
  B. 2016.
\newblock Abstractive Text Summarization using Sequence-to-sequence RNNs and
  Beyond.
\newblock In \emph{Proceedings of The 20th SIGNLL Conference on Computational
  Natural Language Learning}, 280--290.

\bibitem[{Novikova, Du{\v{s}}ek, and Rieser(2017)}]{novikova2017e2e}
Novikova, J.; Du{\v{s}}ek, O.; and Rieser, V. 2017.
\newblock The E2E Dataset: New Challenges For End-to-End Generation.
\newblock In \emph{Proceedings of the 18th Annual SIGdial Meeting on Discourse
  and Dialogue}, 201--206.

\bibitem[{Ott et~al.(2019)Ott, Edunov, Baevski, Fan, Gross, Ng, Grangier, and
  Auli}]{ott2019fairseq}
Ott, M.; Edunov, S.; Baevski, A.; Fan, A.; Gross, S.; Ng, N.; Grangier, D.; and
  Auli, M. 2019.
\newblock fairseq: A Fast, Extensible Toolkit for Sequence Modeling.
\newblock In \emph{Proceedings of NAACL-HLT 2019: Demonstrations}.

\bibitem[{Papineni et~al.(2002)Papineni, Roukos, Ward, and
  Zhu}]{papineni2002bleu}
Papineni, K.; Roukos, S.; Ward, T.; and Zhu, W.-J. 2002.
\newblock BLEU: a method for automatic evaluation of machine translation.
\newblock In \emph{Proceedings of the Annual meeting on Association for
  Computational Linguistics}, 311--318.

\bibitem[{Radford et~al.(2018)Radford, Narasimhan, Salimans, and
  Sutskever}]{radfordimproving}
Radford, A.; Narasimhan, K.; Salimans, T.; and Sutskever, I. 2018.
\newblock Improving Language Understanding by Generative Pre-Training .

\bibitem[{Radford et~al.(2019)Radford, Wu, Child, Luan, Amodei, and
  Sutskever}]{radfordlanguage}
Radford, A.; Wu, J.; Child, R.; Luan, D.; Amodei, D.; and Sutskever, I. 2019.
\newblock Language Models are Unsupervised Multitask Learners .

\bibitem[{Sennrich, Haddow, and Birch(2016)}]{sennrich2016neural}
Sennrich, R.; Haddow, B.; and Birch, A. 2016.
\newblock Neural Machine Translation of Rare Words with Subword Units.
\newblock In \emph{Proceedings of the 54th Annual Meeting of the Association
  for Computational Linguistics (Volume 1: Long Papers)}, 1715--1725.

\bibitem[{Sutskever, Vinyals, and Le(2014)}]{sutskever2014sequence}
Sutskever, I.; Vinyals, O.; and Le, Q.~V. 2014.
\newblock Sequence to sequence learning with neural networks.
\newblock In \emph{Advances in Neural Information Processing Systems},
  3104--3112.

\bibitem[{Vaswani et~al.(2017)Vaswani, Shazeer, Parmar, Uszkoreit, Jones,
  Gomez, Kaiser, and Polosukhin}]{vaswani2017attention}
Vaswani, A.; Shazeer, N.; Parmar, N.; Uszkoreit, J.; Jones, L.; Gomez, A.~N.;
  Kaiser, {\L}.; and Polosukhin, I. 2017.
\newblock Attention is all you need.
\newblock In \emph{Advances in neural information processing systems},
  5998--6008.

\bibitem[{Vig(2018)}]{vig2018deconstructing}
Vig, J. 2018.
\newblock Deconstructing bert: Distilling 6 patterns from 100 million
  parameters.
\newblock \emph{Medium, December} .

\bibitem[{Welleck et~al.(2020)Welleck, Kulikov, Roller, Dinan, Cho, and
  Weston}]{welleck2020neural}
Welleck, S.; Kulikov, I.; Roller, S.; Dinan, E.; Cho, K.; and Weston, J. 2020.
\newblock Neural Text Generation With Unlikelihood Training.
\newblock In \emph{International Conference on Learning Representations
  (ICLR)}.

\bibitem[{Wiseman, Shieber, and Rush(2017)}]{wiseman2017challenges}
Wiseman, S.; Shieber, S.; and Rush, A. 2017.
\newblock Challenges in Data-to-Document Generation.
\newblock In \emph{Proceedings of the Conference on Empirical Methods in
  Natural Language Processing (EMNLP)}, 2253--2263.

\bibitem[{Wiseman, Shieber, and Rush(2018)}]{wiseman2018learning}
Wiseman, S.; Shieber, S.; and Rush, A. 2018.
\newblock Learning Neural Templates for Text Generation.
\newblock In \emph{Proceedings of the Conference on Empirical Methods in
  Natural Language Processing (EMNLP)}, 3174--3187.

\bibitem[{Wu et~al.(2016)Wu, Schuster, Chen, Le, Norouzi, Macherey, Krikun,
  Cao, Gao, Macherey et~al.}]{wu2016google}
Wu, Y.; Schuster, M.; Chen, Z.; Le, Q.~V.; Norouzi, M.; Macherey, W.; Krikun,
  M.; Cao, Y.; Gao, Q.; Macherey, K.; et~al. 2016.
\newblock Google's neural machine translation system: Bridging the gap between
  human and machine translation.
\newblock \emph{arXiv preprint arXiv:1609.08144} .

\end{thebibliography}
    \clearpage
    \appendix
    \begin{center}
      {\LARGE \textbf{Appendix. Supplementary Material}}
      \end{center}
    
    \renewcommand{\thesection}{A.\arabic{section}}
    \setcounter{theorem}{0}
    \setcounter{lemma}{0}
    \setcounter{corollary}{0}
    
    \section{Proof of Theorem \ref{thm:bound}}\label{sec:proof-thm-bound}

    \begin{theorem}
      If $\zeta n > \rho (B^2)$,
      $$R=\operatorname{tr}(B^2(\zeta n I-B^2)^{-1})\le \frac{\|B^2\|_*}{\sigma_n(\zeta nI-B^2)},$$
    in which $\|\cdot\|_*$ is the nuclear norm; $\sigma_n$ denotes the smallest singular value; $\rho(\cdot)$ is the spectral radius of the matrix; $I$ is the identity matrix.
    \end{theorem}
    
    \begin{proof}
      $$\begin{aligned}
        R&=\operatorname{tr} (\frac{B^2}{\zeta n}) + \operatorname{tr} (\frac{B^4}{\zeta^2 n^2})+\cdots\\
        &=\operatorname{tr}(\frac{B^2}{\zeta n}+\frac{B^4}{\zeta^2 n^2}+\cdots)\\
        &=\operatorname{tr}(B^2(\zeta nI-B^2)^{-1})\\
        &\le \sum_{i=1}^n\sigma_i(B^2)\sigma_i((\zeta nI-B^2)^{-1})\\
        &\le \sum_{i=1}^n\sigma_i(B^2)\sigma_1((\zeta nI-B^2)^{-1})\\
        &= \frac{\|B^2\|_*}{\sigma_n(\zeta nI-B^2)}.
        \end{aligned}$$
    \end{proof}
    
    \section{Proof of Corollary \ref{col:sample}}\label{sec:proof-col-sample}
    
    \begin{corollary}
      \label{col:sample}
    $$R \le \frac{\sqrt{r}(\sum_{i=1}^n \sum_{j=1}^n (B_{ij}-\mu_i)^2 +\sum_{i=1}^n (1-b_i)^2)}{\sigma_n(\zeta nI-B^2)},$$
    where $r$ is the rank of $B^2$ and $\mu_i=\frac{\sum_{k=1}^n B_{ik}}{n}$ is the mean of each row of $B$.
    \end{corollary}
    
    \begin{proof}
      $$\begin{aligned}
      \|B^2\|_* &\le \sqrt{r}\|B^2\|_F \\
      & \le \sqrt{r}\|B\|_F^2 \\
      & = \sqrt{r}\sum_{i=1}^n \sum_{j=1}^n B_{ij}^2 \\
      &= \sqrt{r}\sum_{i=1}^n \sum_{j=1}^n ( B_{ij}^2 - 2\mu_i^2 + 2\mu_i^2) \\
      &= \sqrt{r}\sum_{i=1}^n \sum_{j=1}^n (B_{ij}^2 - 2\mu_i\frac{1}{n}\sum_{k=1}^n B_{ik} + 2\mu_i^2) \\
      &= \sqrt{r}\sum_{i=1}^n \sum_{j=1}^n (B_{ij}^2 - 2\mu_i B_{ij} + \mu_i^2+\mu_i^2)\\
      &= \sqrt{r}\sum_{i=1}^n \sum_{j=1}^n ((B_{ij}-\mu_i)^2 +\mu_i^2) \\
      &= \sqrt{r}(\sum_{i=1}^n \sum_{j=1}^n (B_{ij}-\mu_i)^2 + n^2\sum_{i=1}^n \mu_i^2) \\
      &= \sqrt{r}(\sum_{i=1}^n \sum_{j=1}^n (B_{ij}-\mu_i)^2 +\sum_{i=1}^n (1-b_i)^2), \\
      \end{aligned}$$
      where $\|\cdot\|_F$ is the Frobenius norm and it is a submultiplicative norm; $b_i$ is the $i$th element of the vector $b$.
      \end{proof}

    \section{Proof of Corollary \ref{col:inout}}\label{sec:proof-col-inout}
    We first give a lemma by \citet{johnson1989gersgorin} and then prove Corollary \ref{col:inout} based on it.
    \begin{lemma}
      \label{lemma:Johnson}
      [\cite{johnson1989gersgorin}, Theorem 3] For an n-by-n matrix $A = (a_{ij})$, the smallest singular value of A is bounded below by 
      $$\min_{1\le i \le n} \{|a_{ii}|-\frac{1}{2}( \sum_{j=1,j\ne i}^n |a_{ij}| + \sum_{j=1,j\ne i}^n |a_{ji}|)  \} $$
    \end{lemma}
    
    \begin{corollary}
    If $\zeta n I-B^2$ is a diagonally dominant matrix,
    $$\begin{aligned}
    R\le\frac{\|B^2\|_*}{\min_{1\le i \le n} \{ \frac{1}{2}( \zeta n - \underbrace{\sum_{j=1}^n (B^2)_{ij}}_{outflow})+\frac{1}{2}(\zeta n - \underbrace{\sum_{k=1}^n (B^2)_{ki}}_{inflow}) \}}
    \end{aligned}$$
    \end{corollary}
    
    \begin{proof}
      Since $B$ is a sub-stochastic matrix, $(B^2)_{ij}\in [0,1)$.  $(\zeta n I-B^2)_{ii} > 0, i \in [1,n]$ and $(\zeta n I-B^2)_{ij} < 0, i\ne j$. $|(\zeta n I-B^2)_{ii}|-\sum_{j=1,j\ne i}^n |(\zeta n I-B^2)_{ji}|=\zeta n - \sum_{j=1}^n (B^2)_{ij}$. By Lemma \ref{lemma:Johnson}, we have 
      $$\begin{aligned}
      &\sigma_n(\zeta n I-B^2)\ge \\
      &\min_{1\le i \le n} \{ \frac{1}{2}( \zeta n - \sum_{j=1}^n (B^2)_{ij})+\frac{1}{2}( \zeta n - \sum_{k=1}^n (B^2)_{ki}) \}
      \end{aligned}$$ 
    \end{proof}

    \section{Proof of Theorem \ref{thm:extend}}\label{sec:proof-thm-extend}
    We first prove some simple facts in Lemma \ref{lemma:mul} and then prove Corollary \ref{col:inout} based on it.
    \begin{lemma}
      \label{lemma:mul}
      Let $T_p,T_q$ be independently drawn from the same distribution as described in Definition \ref{def:garp}. Furthermore, let $B$ be a sub-stochastic matrix with $\sum_{i=1}^n B_{ij}^2<1$. We have the following facts:
      $$\begin{aligned}
        &\mathbb{E}[(T_pB)_{ij}] = 0; \mathbb{E}[(BT_p)_{ij}] = 0; \mathbb{E}[(T_pT_q)_{ij}] = 0;\\
        &\operatorname{Var}[(T_pB)_{ij}]\le \delta^2;\operatorname{Var}[(BT_p)_{ij}]\le \delta^2;\operatorname{Var}[(T_pT_q)_{ij}]\le \delta^2; 
      \end{aligned}
      $$
    \end{lemma}
    
    \begin{proof}
    $$\begin{aligned}
      \mathbb{E}[(T_pB)_{ij}]=\mathbb{E}[\sum_{k=1}^n(T_p)_{ik}B_{kj}]=\sum_{k=1}^n\mathbb{E}[(T_p)_{ik}]B_{kj}=0
    \end{aligned}$$
    $$\begin{aligned}
      \mathbb{E}[(BT_p)_{ij}]=\mathbb{E}[\sum_{k=1}^nB_{ik}(T_p)_{kj}]=\sum_{k=1}^nB_{ik}\mathbb{E}[(T_p)_{kj}]=0
    \end{aligned}$$
    $$\begin{aligned}
      \mathbb{E}[(T_pT_q)_{ij}]=\mathbb{E}[\sum_{k=1}^n(T_p)_{ik}(T_q)_{kj}]=\sum_{k=1}^n\mathbb{E}[(T_p)_{kj}]\mathbb{E}[(T_q)_{kj}]=0
    \end{aligned}$$
    
    $$\begin{aligned}
      \operatorname{Var}[(T_pB)_{ij}]&=\operatorname{Var}[\sum_{k=1}^n(T_p)_{ik}B_{kj}]
      =\sum_{k=1}^nB_{kj}^2\operatorname{Var}[(T_p)_{ik}]
      \le \delta^2
      \end{aligned}$$
    
    $$\begin{aligned}
      \operatorname{Var}[(BT_p)_{ij}]=\operatorname{Var}[\sum_{k=1}^nB_{ik}(T_p)_{kj}]
      =\sum_{k=1}^nB_{ik}^2\operatorname{Var}[(T_p)_{kj}]
      \le \delta^2
    \end{aligned}$$
    
    $$\begin{aligned}
      \operatorname{Var}[(T_pT_q)_{ij}]&=\operatorname{Var}[\sum_{k=1}^n(T_p)_{ik}(T_q)_{kj}]=n\delta^4\le \delta^2
    \end{aligned}$$
    
    \end{proof}
    
    \begin{theorem}
      \label{thm:extend}
      For a general generation model, if $\sum_{i=1}^n B_{ij}^2<1,\zeta n>4$, then for every constant $a>0$, we have $$\operatorname{Pr}(|R-R'|\ge a)\le \frac{3\zeta n\delta^2 }{a^2(\zeta n-4)(\zeta n-1)}.$$
    \end{theorem}
    
    \begin{proof}
      $$\begin{aligned}
        R'-R&=\sum_{r=1}^\infty \operatorname{tr}(\frac{\prod_{k=1}^{2r}B'_k}{\zeta^r n^{r}})-\sum_{r=1}^\infty \operatorname{tr}(\frac{B^{2r}}{\zeta^r n^{r}})\\
        &=\operatorname{tr}(\frac{(B+T_1)(B+T_2)}{\zeta n}+\frac{\prod_{k=1}^4(B+T_k)}{\zeta^2 n^2}+\cdots)-\\ & \ \ \ \ \ \ \ \ \ \ \ \operatorname{tr}(\frac{B^2}{\zeta n}+\frac{B^4}{\zeta^2 n^2}+\cdots)\\
          &=\operatorname{tr}(\frac{T_1B+BT_2+T_1T_2}{\zeta n}+\frac{B^2T_3B+B^3T_4+B^2T_3T_4}{\zeta^2 n^2}+\\
          &\ \ \ \ \frac{T_1B^3+T_1BT_3B+T_1B^2T_4+T_1BT_3T_4}{\zeta^2 n^2}+\\
          &\ \ \ \ \frac{BT_2B^2+BT_2T_3B+BT_2BT_4+BT_2T_3T_4}{\zeta^2 n^2}+ \\
          &\ \ \ \ \frac{T_1T_2B^2+T_1T_2T_3B+T_1T_2BT_4+T_1T_2T_3T_4}{\zeta^2 n^2}+\cdots)
      \end{aligned}$$
      By Lemma \ref{lemma:mul}, we can get:
      $$\begin{aligned}
        \mathbb{E}[R'-R]=0
      \end{aligned}
      $$
      
      $$\begin{aligned}
        \operatorname{Var}[R'-R]&\le (\frac{3n\delta^2}{\zeta n}+\frac{15n\delta^2}{\zeta^2 n^2}+\cdots)=n\sum_{k=1}^\infty \frac{(4^k-1)\delta^2}{\zeta^k n^k}\\
      &=\frac{3\zeta n\delta^2}{(\zeta n-4)(\zeta n-1)}
      \end{aligned}$$
      
      By Chebyshev's inequality, we have:
      
      $$\begin{aligned}
        \operatorname{Pr}(|R'-R-\mathbb{E}(R'-R)|\ge a)&\le \frac{\operatorname{Var}(R'-R)}{a^2}\\
        \operatorname{Pr}(|R'-R|\ge a)&\le \frac{3\zeta n\delta^2}{a^2(\zeta n-4)(\zeta n-1)}
      \end{aligned}$$
      \end{proof}

    \section{Human Evaluation Questionnaire}\label{sec:appendix-questionnaire}

    This is a sample of human evaluation questionnaire. It contains a sample text and the human helpers are asked to answer two questions by scoring between $0\sim 5$.
    
    \noindent
    \fbox{\parbox{\linewidth}{Read the sentences below and use the sliders below indicate the overall performance and the repetition score. (1 = very poor, 5 = very good)
    \\
    \\
    Sentence : The Beatles were more absurd than their previous stories . They were often compared to the audience ; they were sometimes portrayed as paraneatingly aggressive and intriguing characters . The Beatles were often seen as a performer of the Beatles , particularly as the Beatles ' " western " side , while by the end of 1969 they were frequently seen as " a black and white band . " Lennon recalled that McCartney was " always a little bit bamboo " . John Lennon , McCartney , McCartney and his friend George Harrison later recalled that " they had it for a few years and a lot of people . "
    \\
    \\
    1) What's the overall quality of the sentence?\\
    (1 = The sentence is incoherent, has many grammatical errors, and difficult to understand,
    5 = The sentence is coherent, grammatically correct, and easy to understand)
    \\
    \\
    
    2) Does the text repeat itself?\\
    (1 = It's always repeating some words and content,
    5 = It never repeats anything.)
    }}

    \section{Hyper-Parameters Tunning}\label{sec:appendix-hyperpars}
    \textbf{NMT Task.} For the Greedy sampling method and the Stochastic sampling method, there is no hyper-parameter to tune. For other models, we tune the temperature parameter with grid search on the valid set. We choose the model with the best BLEU score from models that has lower repetition scores than the Greedy method. For the Temperature sampling method, we set temperature $t=0.15$ by searching from [0.1, 0.15, 0.2, 0.25, 0.3, 0.35, 0.4, 0.45, 0.5, 0.55, 0.6, 0.65, 0.7, 0.75, 0.8, 0.85, 0.9, 0.95, 1.0]. For the Topk/Nucleus sampling method, we report Topk sampling with $k=10/40$ and Nucleus sampling with $p=0.9/0.95$ following the setting of \citet{holtzman2020curious}. The temperature for them are chosen similar to the Temperature sampling method and we set the temperature $t=0.15$ for Topk ($k$=10), $t=0.2$ for Topk ($k$=40), $t=0.2$ for Nucleus ($p$=0.9), $t=0.15$ for Nucleus ($p$=0.95). For the LP method, we choose $\beta=6$ from [3,4,5,6,7,8]. For our proposed RE method, we do not need to tune the temperature parameter since it already alleviates the repetition problem a lot. We follow the default setting in \cite{ott2019fairseq} and run the program until the BLEU score on the valid set does not change too much. We use the Adam optimization method with the learning rate set to 5e-5. We use label smoothed cross-entropy with label smooth parameter set to 0.1. We set the dropout to 0.3.
    
    \textbf{LM Task.} For the Greedy method and the Stochastic sampling method, there is no hyper-parameter to tune. For other models, we tune the temperature parameter with grid search. We choose the model with the ppl-c score most close to 1.0 since ppl-c for human written language is 1.0. For the Temperatue sampling method, we set temperature $t=0.75$ by searching from [0.3, 0.35, 0.4, 0.45, 0.5, 0.55, 0.6, 0.65, 0.7, 0.75, 0.8, 0.85, 0.9, 0.95, 1.0]. We use similar method to set $t=0.9$ when $k=40$ and $t=1.0$ when $k=10$ for the Topk sampling method. We set $t=0.85$ when $p=0.9$ and $p=0.95$ for the Nucleus method. We choose $\beta=7$ from [3,4,5,6,7,8]. We set $t=0.75$ when $\gamma=0.1$ and $t=0.7$ when $\gamma=0.08$ for the RE model. We follow the default setting in \cite{ott2019fairseq} and run the program until the perplexity score does not change too much. We use the Adam optimization method with learning rate set to 5e-5 and we set the dropout to 0.1.
\end{document}